\documentclass{article}

\usepackage{microtype}
\usepackage{graphicx}
\usepackage{subfigure}
\usepackage{booktabs}

\usepackage{times}
\usepackage{natbib}

\usepackage{algorithm}
\usepackage{algorithmic}

\usepackage{helvet}
\usepackage{courier}
\usepackage{amsfonts,amssymb}
\usepackage{amsmath}
\usepackage[utf8]{inputenc}
\usepackage[T1]{fontenc}
\usepackage{url}
\usepackage{nicefrac}
\usepackage{enumitem}
\usepackage{amsthm}
\usepackage{amssymb}

\newtheorem{theorem}{Theorem}[section]

\newtheorem{proposition}[theorem]{Proposition}

\usepackage{hyperref}

\usepackage[accepted]{icml2018}
\usepackage{comment}
\usepackage{todonotes}
\icmltitlerunning{Towards Binary-Valued Gates for Robust LSTM Training}

\icmlsetsymbol{equal}{*}
\begin{document}
\twocolumn[
\icmltitle{Towards Binary-Valued Gates for Robust LSTM Training}
\begin{icmlauthorlist}
\icmlauthor{Zhuohan Li}{pkumoe}
\icmlauthor{Di He}{pkumoe}
\icmlauthor{Fei Tian}{msr}
\icmlauthor{Wei Chen}{msr}
\icmlauthor{Tao Qin}{msr}
\icmlauthor{Liwei Wang}{pkumoe,pkucds}
\icmlauthor{Tie-Yan Liu}{msr}

\end{icmlauthorlist}

\icmlaffiliation{pkumoe}{Key Laboratory of Machine Perception, MOE, School of EECS, Peking University}
\icmlaffiliation{pkucds}{Center for Data Science, Peking University, Beijing Institute of Big Data Research}
\icmlaffiliation{msr}{Microsoft Research}
\icmlcorrespondingauthor{Tao Qin}{taoqin@microsoft.com}
\icmlkeywords{Recurrent Neural Network, LSTM, Long-Short Term Memory Network, Machine Translation}

\vskip 0.3in
]
\printAffiliationsAndNotice{The work was done while the first author was visiting Microsoft Research Asia.}

\begin{abstract}
Long Short-Term Memory (LSTM) is one of the most widely used recurrent structures in sequence modeling. It aims to use gates to control information flow (e.g., whether to skip some information or not) in the recurrent computations, although its practical implementation based on soft gates only partially achieves this goal. In this paper, we propose a new way for LSTM training, which pushes the output values of the gates towards 0 or 1. By doing so, we can better control the information flow: the gates are mostly open or closed, instead of in a middle state, which makes the results more interpretable. Empirical studies show that (1) Although it seems that we restrict the model capacity, there is no performance drop: we achieve better or comparable performances due to its better generalization ability; (2) The outputs of gates are not sensitive to their inputs: we can easily compress the LSTM unit in multiple ways, e.g., low-rank approximation and low-precision approximation. The compressed models are even better than the baseline models without compression. 
\end{abstract}

\begin{figure*}[htb]
\centering	
\begin{minipage}{0.30\linewidth}
\subfigure[Input gates in LSTM]{
\includegraphics[width = 1\linewidth]{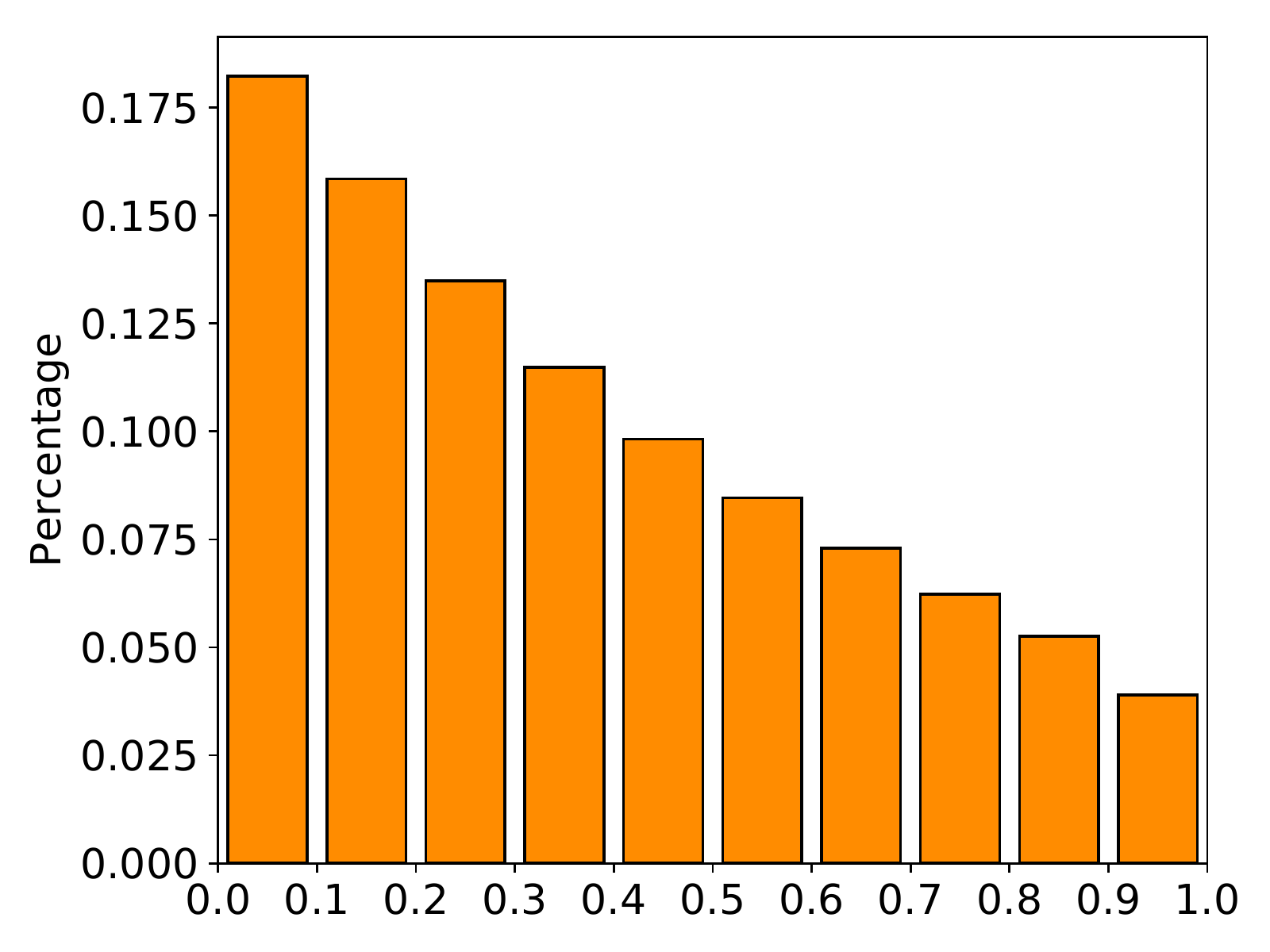}
}
\end{minipage}%
\qquad
\begin{minipage}{0.30\linewidth}
\subfigure[Forget gates in LSTM]{
\includegraphics[width = 1\linewidth]{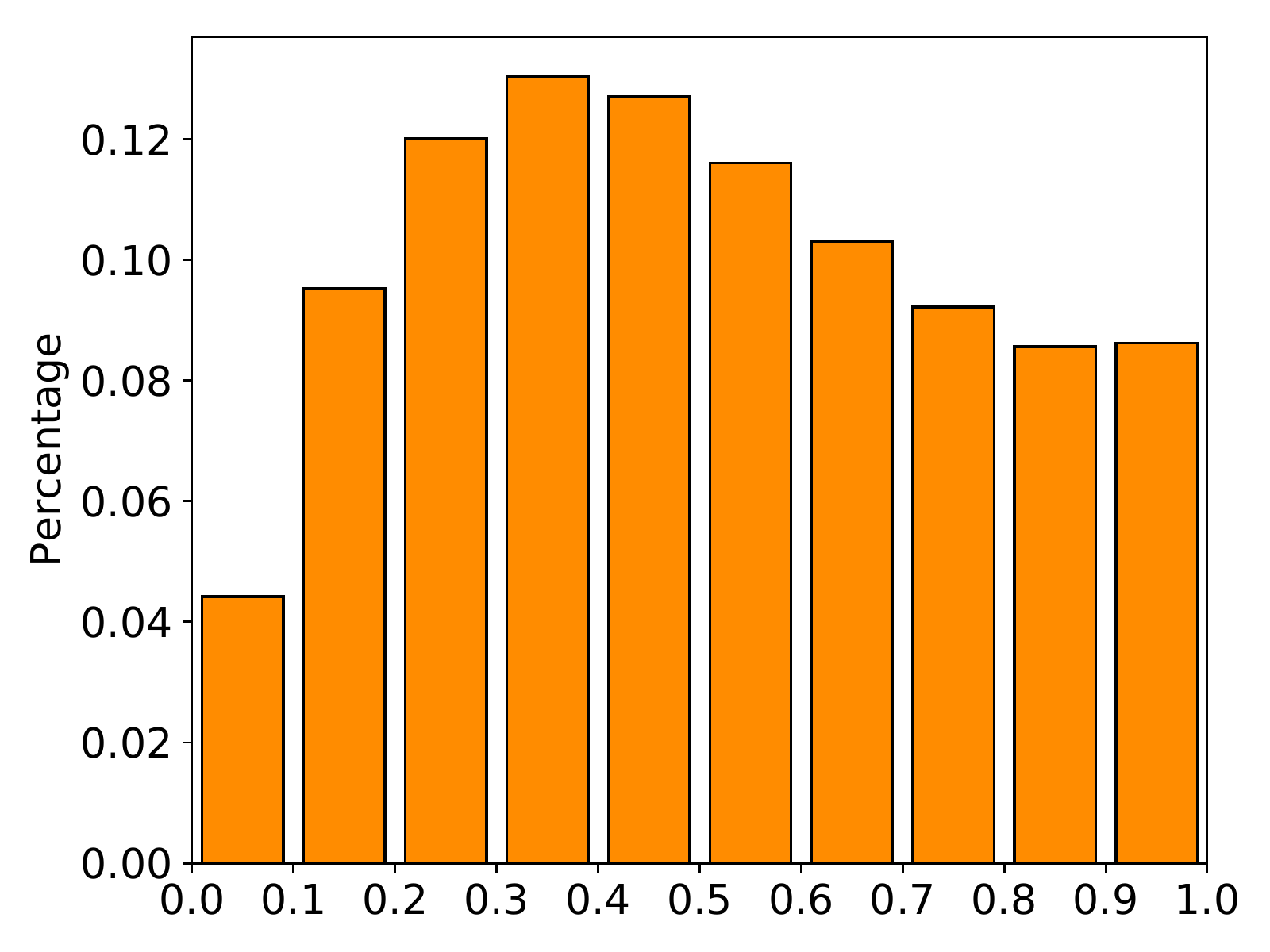}
}
\end{minipage}%
\caption{Histograms of gate value distributions in LSTM, based on the gate outputs of the first-layer LSTM in the decoder from 10000 sentence pairs IWSLT14 German$\rightarrow$English training sets.}
\label{fig:gate_lstm}
\end{figure*}

\section{Introduction}
Recurrent neural networks (RNNs) \citep{hochreiter1998vanishing} are widely used in sequence modeling tasks, such as language modeling \citep{kim2016character,jozefowicz2016exploring}, speech recognition \citep{zhang2016highway}, time series prediction \citep{xingjian2015convolutional}, machine translation \citep{wu2016google,britz2017massive,he2016dual}, image captioning \citep{vinyals2015show,xu2015show}, and image generation \citep{villegas2017learning}.

To address the long-term dependency and gradient vanishing problem of conventional RNNs, long short-term memory (LSTM) \citep{gers1999learning,hochreiter1997long} networks were proposed, which introduce \emph{gate functions} to  control the information flow in a recurrent unit: a \emph{forget gate function} to determine how much previous information should be excluded for the current step, an \emph{input gate function} to find relevant signals to be absorbed into the hidden context, and an \emph{output gate function} for prediction and decision making. For ease of optimization, in practical implementation, one usually uses the element-wise sigmoid function to mimic the gates, whose outputs are soft values between 0 and 1.

By using such gates with many more parameters, LSTM usually performs much better than conventional RNNs. However, when looking deep into the unit, we empirically find that the values of the gates are not that meaningful as the design logic. For example, in Figure \ref{fig:gate_lstm}, the distributions of the forget gate values and input gate values are not sharp and most of the values are in the middle state (around 0.5), meaning that most of the gate values are ambiguous in LSTM. This phenomenon contradicts the design of both gates: to control whether or not to take the information from the previous timesteps or the new inputs. At the same time, several works \citep{murdoch2017automatic, karpathy2015visualizing} show that most cell coordinates of LSTM are hard to find particular meanings.

In this paper, we propose to push the values of the gates to the boundary of their ranges $(0,1)$.\footnote{The output of a gate function is usually a vector. For simplicity, in the paper, we say ``pushing the output of the gate function to 0/1'' when meaning ``pushing each dimension of the output vector of the gate function to either 0 or 1''. We also say that each dimension of the output vector of the gate function is a gate, and say a gate is open/closed if its value is close to 1/0.} Pushing the values of the gates to 0/1 has certain advantages. First, it well aligns with the original purpose of the development of gates: to get the information in or skip by ``opening'' or ``closing'' the gates during the recurrent computation, which reflects more accurate and clear linguistic and structural information.  Second, similar to BitNet in image classification \cite{courbariaux2016binarized}, by pushing the activation function to be binarized, we can learn a model that is ready for further compression. Third, training LSTM towards binary-valued gates enables better generation of the learned model. According to \citep{hochreiter1997flat,haussler1997mutual,keskar2016large,chaudhari2016entropy},  a model lying in a flat region of the loss surface is likely to generalize well, since any small perturbation to the model makes little fluctuation to the loss. Training LSTM towards binary-valued gates means seeking a set of parameters to make the values of the gates approaching zero or one, namely residing in the flat region of the sigmoid function, which corresponds to the flat region of the overall loss surface.

Technically, pushing the outputs of the gates towards such discrete values is challenging. A straightforward approach is to sharpen the sigmoid function by a small temperature. However, this is equivalent to rescaling the input and cannot guarantee the values of the learned gates to be close to 0 or 1. To tackle this challenge, in this paper, we leverage the Gumbel-Softmax estimator developed for variational methods \cite{jang2016categorical,maddison2016concrete}. The estimator generates approximated and differentiable samples for categorical latent variables in a stochastic computational graph, e.g., variational autoencoder. Specifically, during training, we apply the Gumbel-Softmax estimator to the gates to approximate the values sampled from the Bernoulli distribution given by the parameters, and train the LSTM model with standard backpropagation methods. We call the learned model Gumbel-Gate LSTM ($G^2$-LSTM).  We conduct experiments on language modeling and machine translation to verify our proposed method. We have the following observations from experimental results:
\begin{itemize}
\item Our method restricts the gate outputs to be close to the boundary, and thus reduces the representation power. Surprisingly, there is no performance drop. Furthermore, our model achieves better or comparable results compared to the baseline model.
\item Our learned model is easy for further compression. We apply several model compression algorithms to the parameters in the gates, including low-precision approximation and low-rank approximation, and results show that our compressed model can be even better than the baseline model without compression.
\item We investigate a set of samples and find that the gates in our learned model are meaningful and intuitively interpretable. We show our model can automatically learn the boundaries in the sentences.
\end{itemize}
The organization of the paper is as follows. We review related work in Section 2 and propose our learning algorithm in Section 3. Experiments are reported in Section 4 and future work is discussed in the last section.
\section{Background}
\subsection{Gumbel-Softmax Estimator}
\label{app:gumbel}
\citet{jang2016categorical} and \citet{maddison2016concrete} develop a continuous relaxation of discrete random variables in stochastic computational graphs. The main idea of the method is that the multinomial distribution can be represented according to Gumbel-Max trick, thus can be approximated by Gumbel-Softmax distribution. In detail, given a probability distribution over $k$ categories with parameter $\pi_1,\pi_2,\ldots,\pi_k$, the Gumbel-Softmax estimator gives an approximate one-hot sample $y$ with
\begin{eqnarray}
y_i = \frac{\exp((\log\pi_i+q_i)/\tau)}{\sum_{j=1}^k\exp((\log\pi_j+q_j)/\tau)} \quad \text{for }i=1,\ldots,k,
\end{eqnarray}
where $\tau$ is the temperature and $q_i$ is independently sampled from Gumbel distribution: $ q_i=-\log(-\log U_i), U_i\sim \text{Uniform}(0,1). $

By using the Gumbel-Softmax estimator, we can generate sample $y=(y_1,...,y_k)$ to approximate the categorical distribution. Furthermore, as the randomness $q$ is independent of $\pi$ (which is usually defined by a set of parameters), we can use reparameterization trick to optimize the model parameters using standard backpropagation algorithms. Gumbel-Softmax estimator has been adopted in several applications such as variation autoencoder \citep{jang2016categorical}, generative adversarial network \citep{kusner2016gans}, and language generation \citep{subramanian2017adversarial}. To the best of our knowledge, this is the first work to introduce the Gumbel-Softmax estimator in LSTM for robust training purpose.

\subsection{Loss surface and generalization}
The concept of sharp and flat minima has been first discussed in \citep{hochreiter1997flat,haussler1997mutual}. Intuitively, a flat minimum $x$ of a loss $f(\cdot)$ corresponds to the point for which the value of function $f$ varies slowly in a relatively large neighborhood of $x$. In contrast, a sharp minimum $x$ is such that the function $f$ changes rapidly in a small neighborhood of $x$. The sensitivity of the loss function at sharp minima negatively impacts the generalization ability of a trained model on new data. Recently, several papers discuss how to modify the training process and to learn a model in a flat region so as to obtain better generalization ability. \citet{keskar2016large} show by using small-batch training, the learned model is more likely to converge to a flat region rather than a sharp one. \citet{chaudhari2016entropy} propose a new objective function considering the local entropy and push the model to be optimized towards a wide valley.

\section{The Proposed Training Algorithm}
In this section, we present a new and robust training algorithm for LSTM by learning towards binary-valued gates.

\subsection{Long Short-Term Memory RNN}
Recurrent neural networks process an input sequence $\{x_1, x_2, \ldots, x_T\}$ sequentially and construct a corresponding sequence of hidden states/representations $\{h_1,h_2,\ldots,h_T\}$. In single-layer recurrent neural networks, the hidden states $\{h_1,h_2,\ldots,h_T\}$ are used for prediction or decision making. In deep (stacked) recurrent neural networks, the hidden states in layer $k$ are used as inputs to layer $k+1$.

In recurrent neural networks, each hidden state is trained (implicitly) to remember and emphasize task-relevant aspects of the preceding inputs, and to incorporate new inputs via a recurrent operator, $T$, which converts the previous hidden state and present input into a new hidden state, e.g.,
\[
h_t = T(h_{t-1}, x_t) = \tanh(W_hh_{t-1}+ W_xx_t + b),
\]
 where $W_h$, $W_x$ and $b$ are parameters.

Long short-term memory RNN (LSTM) \citep{ hochreiter1997long} is a carefully designed recurrent structure. In addition to the hidden state $h_t$ used as a transient representation of state at timestep $t$, LSTM introduces a memory cell $c_t$, intended for internal long-term storage.
$c_t$ and $h_t$ are computed via three gate functions. The forget gate function $f_t$ directly connects $c_t$ to the memory cell $c_{t-1}$ of the previous timestep via an element-wise multiplication. Large values of the forget gates cause the cell to remember most (if not all) of its previous values. The other gates control the flow of information in input ($i_t$) and output ($o_t$) of the cell. Each gate function has a weight matrix and a bias vector; we use subscripts $f$, $i$ and $o$ to denote parameters for the forget gate function, the input gate function and the output gate function respectively, e.g., the parameters for the forget gate function are denoted by  $W_{xf}, W_{hf},$ and $b_f$.

With the above notations, an LSTM is formally defined as
\begin{eqnarray}
i_t &=& \sigma(W_{xi}x_t + W_{hi}h_{t-1} + b_i),\label{eqn:lstmstart}\\
f_t &=& \sigma(W_{xf}x_t + W_{hf}h_{t-1} + b_f),\\
o_t &=& \sigma(W_{xo}x_t + W_{ho}h_{t-1} + b_o),\\
g_t &=& \tanh(W_{xg}x_t + W_{hg}h_{t-1} + b_g),\\
c_t &=& f_t\odot c_{t-1} + i_t\odot g_t,\\
h_t &=& o_t\odot \tanh(c_t),\label{eqn:lstmend}
\end{eqnarray}
where $\sigma(\cdot)$ represents the sigmoid function and $\odot$ is the element-wise product.

\subsection{Training LSTM Gates Towards Binary Values}
\label{ssec:training}
The LSTM unit requires much more parameters than the simple RNN unit. As we can see from Eqn. (\ref{eqn:lstmstart}) - (\ref{eqn:lstmend}), a large percentage of the parameters are used to compute the gate (sigmoid) functions. If we can push the outputs of the gates to the saturation area of the sigmoid function (i.e., towards 0 or 1), the loss function with respect to the parameters in the gates will be flat: if the parameters in the gates perturb, the change to the output of the gates is small due to the sigmoid operator (see Figure \ref{fig:sigmoid}), and then the change to the loss is little, which means the flat region of the loss. First, as such model is robust to small parameter changes, it is robust to different model compression methods, e.g., low-precision compression or low-rank compression. Second, as discussed in \citep{chaudhari2016entropy}, minima in a flat region is more likely to generalize better, and thus toward binary-valued gates may lead to better test performance.

However, the task of training towards binary-valued gates is quite challenging. One straightforward idea is to sharpen the sigmoid function by using a smaller temperature, i.e., $f_{W,b}(x)= \sigma((Wx+b)/\tau)$, where $\tau<1$ is the temperature. However, it is computationally equivalent to $f_{W^\prime,b^\prime}(x) =\sigma(W^\prime x+b^\prime)$ by setting $W^\prime=W/\tau$ and $b^\prime = b/\tau$. Then using a small temperature is equivalent to rescale the initial parameters as well as the gradients to a larger range. Usually, using an initial point in a large range with a large learning rate will harm the optimization process, and apparently cannot guarantee the outputs to be close to the boundary after training.
\begin{figure}
  \centering
    \includegraphics[width=0.5\textwidth]{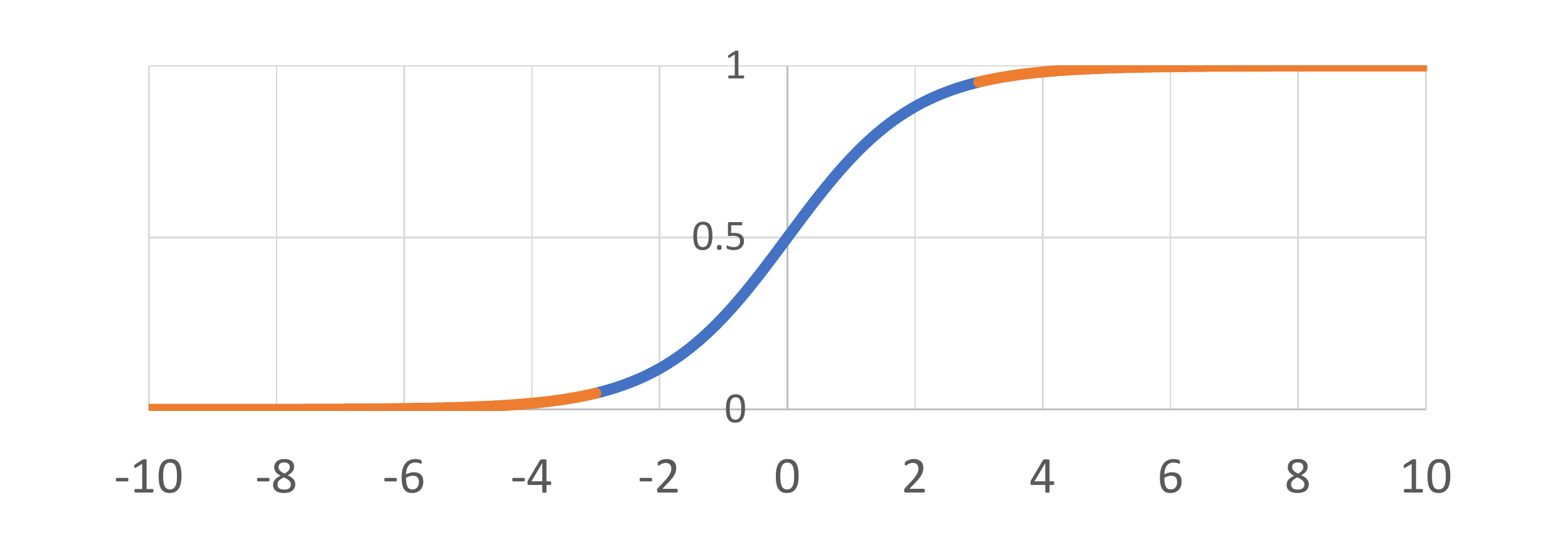}
\caption{
The orange parts correspond to the saturation area of the sigmoid function.
}\label{fig:sigmoid}
\end{figure}

In this work, we leverage the recently developed Gumbel-Softmax trick. This trick is efficient in approximating discrete distributions, and is one of the widely used methods to learn discrete random variables in stochastic computational graphs. We first provide a proposition about the approximation ability of this trick for Bernoulli distribution, which will be used in our proposed algorithm.
\begin{proposition}
	Assume $\sigma(\cdot)$ is the sigmoid function. Given $\alpha\in \mathbb{R}$ and temperature $\tau>0$, we define random variable $D_{\alpha}\sim B(\sigma(\alpha))$ where $B(\sigma(\alpha))$ is the Bernoulli distribution with parameter $ \sigma(\alpha)$, and define  $G(\alpha,\tau)=\sigma\left(\frac{\alpha + \log U-\log(1-U)}{\tau}\right)$ where $U\sim \text{Uniform}(0,1)$. Then the following inequalities hold for arbitrary $ \epsilon\in (0,{1}/{2})$,
\begin{align}
  P(D_{\alpha}=1)-({\tau}/{4})\log({1}/{\epsilon})\le{}& P(G(\alpha,\tau)\ge 1-\epsilon) \nonumber\\
 \le{}& P(D_{\alpha}=1), \label{eqn:prop1}\\
 P(D_{\alpha}=0)-({\tau}/{4})\log({1}/{\epsilon})\le{}&P(G(\alpha,\tau) \le \epsilon)\nonumber\\
 \le{}&  P(D_{\alpha}=0). \label{eqn:prop2}
\end{align}
\end{proposition}
\begin{proof}
Since $\sigma^{-1}(x)=\log\left(\frac{x}{1-x}\right)$, we have 
\begin{align*}
&P(G(\alpha,\tau)\ge 1-\epsilon) \\
={}&P\left(\frac{\alpha + \log U-\log(1-U)}{\tau}\ge \log(1/{\epsilon}-1)\right) \\
={}&P(e^{\alpha -\tau\log({1}/{\epsilon}-1)} \ge {(1-U)}/{U}) \\
={}&P\left(U\ge \frac{1}{1+e^{\alpha -\tau\log({1}/{\epsilon}-1)}}\right) \\
={}&\sigma(\alpha -\tau\log({1}/{\epsilon}-1)).
\end{align*}
Considering that sigmoid function is $({1}/{4})$-Lipschitz continuous and morotonically increasing, we have 
\begin{align*}
&P(D_{\alpha}=1)-P(G(\alpha,\tau)\ge 1-\epsilon) \\
={}& \sigma(\alpha)-\sigma(\alpha -\tau\log({1}/{\epsilon}-1)) \\
\le{}& ({\tau}/{4})\log({1}/{\epsilon}-1) \le  ({\tau}/{4})\log({1}/{\epsilon})
\end{align*}
and $ P(D_{\alpha}=1)-P(G(\alpha,\tau)\ge 1-\epsilon)\ge 0$. We omit the proof for Eqn.~(\ref{eqn:prop2}) as it is almost identical to the proof of Eqn.~(\ref{eqn:prop1}).
\end{proof}
We can see from the above proposition, the distribution of $G(\alpha,\tau)$ can be considered as an approximation of Bernoulli distribution $B(\sigma(\alpha))$. The rate of convergence is characterized by Eqn.~(\ref{eqn:prop1}) and (\ref{eqn:prop2}). When the temperature $\tau$ approaches positive zero, we directly obtain the following property, which is also proved by \citet{maddison2016concrete},
\begin{eqnarray}
P\left(\lim_{\tau\rightarrow 0^+}G(\alpha,\tau)=1\right)&=& P(D_{\alpha}=1), \nonumber\\
P\left(\lim_{\tau\rightarrow 0^+}G(\alpha,\tau)=0\right)&=& P(D_{\alpha}=0).
\end{eqnarray}
We apply this method into the computation of the gates. Imagine a one-dimensional gate $\sigma(\alpha(\theta))$ where $\alpha$ is a scalar parameterized by $\theta$, and assume the model will produce a larger loss if the output of the gate is close to one, and produce a smaller loss if the gate value is close to zero. If we can repeatedly sample the output of the gate using $G(\alpha(\theta),\tau)=\sigma\left(\frac{\alpha(\theta) + \log U-\log(1-U)}{\tau}\right)$ and estimate the loss, any gradient-based algorithm will push the parameter $\theta$ such that the output value of the gate is close to zero in order to minimize the expected loss. By this way, we can optimize towards the binary-valued gates.

As the gate function is usually a vector-valued function, we extend the notations into a general form: Given $\alpha\in \mathbb{R}^d$ and $\tau>0$, we define $G(\alpha,\tau)=\sigma\left(\frac{\alpha + \log U-\log(1-U)}{\tau}\right)$, where $U$ is a vector and each element $u_i$ in $U$ is independently sampled from $\text{Uniform}(0,1)$, $i=1,2,\ldots,d$. 

In particular, we only push the outputs of input gates and forget gates towards binary values as the output gates usually need fine-granularity information for decision making which makes binary values less desirable. To justify this, we conducted similar experiments and observed a performance drop when pushing the output gates to 0/1 together with the input gates and the forget gates.

We call our proposed learning method Gumbel-Gate LSTM ($G^2$-LSTM), which works as follows during training:
\begin{eqnarray}
i_t &=& G (W_{xi}x_t + W_{hi}h_{t-1} + b_i, \tau) \label{eqn:g2lstmstart}\\
f_t &=& G (W_{xf}x_t + W_{hf}h_{t-1} + b_f, \tau)\\
o_t &=& \sigma(W_{xo}x_t + W_{ho}h_{t-1} + b_o)\\
g_t &=& \tanh(W_{xg}x_t + W_{hg}h_{t-1} + b_g)\\
c_t &=& f_t\odot c_{t-1} + i_t\odot g_t\\
h_t &=& o_t\odot \tanh(c_t). \label{eqn:g2lstmend}
\end{eqnarray}
In the forward pass, we first independently sample values for $U$ in each time step, then update $G^2$-LSTMs using Eqn. (\ref{eqn:g2lstmstart}) - (\ref{eqn:g2lstmend}) and calculate the loss, e.g., negative log likelihood loss. In the backward pass, as $G$ is continuous and differentiable with respect to the parameters and the loss is continuous and differentiable with respect to $G$, we can use any standard gradient-based method to update the model parameters.

\begin{table*}[tb]
\centering
\caption{Performance comparison on language model (perplexity)}\label{exp:performancelm}
\centerline{
\begin{tabular}{|l|l|l|l|}
\hline
Model                                                               & Size      & Valid         & Test           \\\hline\hline
\multicolumn{4}{|l|}{\emph{Existing results} }\\\hline
Unregularzed LSTM                                                   & 7M        & 120.7         & 114.5          \\
NR-dropout                          \citep{zaremba2014recurrent}    & 66M       & 82.2          & 78.4           \\
Zoneout                             \citep{krueger2016zoneout}      & 66M       & -             & 77.4           \\
Variational LSTM                    \citep{gal2016theoretically}    & 19M       & -             & 73.4           \\
CharCNN                             \citep{kim2016character}        & 21M       & 72.4          & 78.9           \\
Pointer Sentinel-LSTM               \citep{merity2016pointer}       & 51M       & -             & 70.9           \\
LSTM + continuous cache pointer     \citep{grave2016improving}      & -         & -             & 72.1           \\
Variational LSTM + augmented loss   \citep{inan2016tying}           & 51M       & 71.1          & 68.5           \\
Variational RHN                     \citep{zilly2016recurrent}      & 23M       & 67.9          & 65.4           \\
NAS Cell                            \citep{zoph2016neural}          & 54M       & -             & 62.4           \\
4-layer skip connection LSTM        \citep{melis2017state}          & 24M       & 60.9          & 58.3           \\
AWD-LSTM w/o finetune                \citep{merity2017regularizing}  & 24M       & 60.7          & 58.8           \\
AWD-LSTM (Baseline)                           \citep{merity2017regularizing}  & 24M       & 60.0          & 57.3           \\\hline\hline
\multicolumn{4}{|l|}{\emph{Our system} }\\\hline
Sharpened Sigmoid AWD-LSTM w/o finetune                               & 24M       & 61.6          & 59.4           \\
Sharpened Sigmoid AWD-LSTM                                            & 24M       & 59.9          & 57.5           \\
$G^2$-LSTM w/o finetune                                             & 24M       & \textbf{60.4} & \textbf{58.2}  \\
$G^2$-LSTM                                                          & 24M       & \textbf{58.5} & \textbf{56.1}  \\\hline\hline
\multicolumn{4}{|l|}{\emph{+continuous cache pointer} }\\\hline
AWD-LSTM + continuous cache pointer \citep{merity2017regularizing}  & 24M       & 53.9          & 52.8           \\
Sharpened Sigmoid AWD-LSTM + continuous cache pointer                 & 24M       & 53.9          & 53.2           \\
$G^2$-LSTM + continuous cache pointer                               & 24M       & \textbf{52.9} & \textbf{52.1}  \\\hline
\end{tabular}
}
\end{table*}

\section{Experiments}
\subsection{Settings}
We tested the proposed training algorithm on two tasks -- language modeling and machine translation.\footnote{Codes for the experiments are available at \url{https://github.com/zhuohan123/g2-lstm}}
\subsubsection{Language Modeling}
Language modeling is a very basic task for LSTM. We used the Penn Treebank corpus that contains about 1 million words. The task is to train an LSTM model to correctly predict the next word conditioned on previous words. A model is evaluated by the prediction perplexity: smaller the perplexity, better the prediction.

We followed the practice in \citep{merity2017regularizing} to set up the model architecture for LSTM: a stacked three-layer LSTM with drop-connect \citep{wan2013regularization} on recurrent weights and a variant of averaged stochastic gradient descent (ASGD) \citep{polyak1992acceleration} for optimization, with a 500-epoch training phase and a 500-epoch finetune phase. Our training code for $G^2$-LSTM was also based on the code released by \citet{merity2017regularizing}. Since the temperature $\tau$ in $G^2$-LSTM does not have significant effects on the results, we set it to 0.9 and followed all other configurations in \citet{merity2017regularizing}. We added neural cache model \citep{grave2016improving} on the top of our trained language model to further improve the perplexity.

\subsubsection{Machine Translation}

\begin{table*}[tb]
\centering
\caption{Performance comparison on machine translation (BLEU)}\label{exp:performance}
\begin{tabular}{|l|l||l|l|}
\hline
English$\rightarrow$German task& BLEU  & German$\rightarrow$English task& BLEU\\\hline\hline
\multicolumn{4}{|l|}{\emph{Existing end-to-end system} }\\\hline
RNNSearch-LV \citep{chousing} &  19.40 & BSO \citep{wiseman2016sequence} & 26.36 \\
MRT \citep{shen2015minimum} &  20.45 & NMPT \citep{huangtoward} &  28.96 \\
Global-att \citep{luong2015effective} &20.90 & NMPT+LM \citep{huangtoward} & 29.16 \\
GNMT \citep{wu2016google}& \textbf{24.61} &ActorCritic \citep{AC4SequencePrediction} &28.53\\\hline\hline
\multicolumn{4}{|l|}{\emph{Our end-to-end system} }\\\hline
Baseline  &  21.89 &-  &  31.00\\
Sharpened Sigmoid  &  21.64 &-  &  29.73 \\
$G^2$-LSTM &22.43 &-  &  \textbf{31.95}\\\hline
\end{tabular}
\end{table*}

We used two datasets for experiments on neural machine translation (NMT): (1) IWSLT’14 German$\rightarrow$English translation dataset~\citep{cettolo2014report}, which is widely adopted in machine learning community~\citep{AC4SequencePrediction,BSO,PG4Sequence}. The training/validation/test sets contain about 153K/7K/7K sentence pairs respectively, with words pre-processed into sub-word units using byte pair encoding (BPE)~\citep{BPE}. We chose 25K most frequent sub-word units as the vocabulary for both German and English. (2) English$\rightarrow$German translation dataset in WMT'14, which is also commonly used as a benchmark task to evaluate different NMT models \citep{bahdanau2014neural, wu2016google, gehring2017convolutional,he2017decoding}. The training set contains 4.5M English$\rightarrow$German sentence pairs, Newstest2014 is used as the test set, and the concatenation of Newstest2012 and Newstest2013 is used as the validation set. Similarly, BPE was used to form a vocabulary of most frequent 30K sub-word units for both languages. In both datasets, we removed the sentences with more than 64 sub-word units in training.

For the German$\rightarrow$English dataset, we adopted a stacked two-layer encoder-decoder framework. We set the size of word embedding and hidden state to 256. As the amount of data in the English$\rightarrow$German dataset is much larger, we adopted a stacked three-layer encoder-decoder framework and set the size of word embedding and hidden state to 512 and 1024 respectively. The first layer of the encoder was bi-directional. We also used dropout in training stacked LSTM as in \citep{zaremba2014recurrent}, with dropout value determined via validation set performance. For both experiments, we set the temperature $\tau$ for $G^2$-LSTM to 0.9, the same as language modeling task. The mini-batch size was 32/64 for German$\rightarrow$English/English$\rightarrow$German respectively. All models were trained with AdaDelta \citep{zeiler2012adadelta} on one M40 GPU. Both gradient clipping norms were set to 2.0. We used tokenized case-insensitive and case-sensitive BLEU as evaluation measure for German$\rightarrow$English/English$\rightarrow$German respectively, following common practice.\footnote{Calculated by the script at \url{https://github.com/moses-smt/mosesdecoder/blob/master/scripts/generic/multi-bleu.perl}} The beam size is set to 5 during the inference step.

\begin{table*}[htb]
\centering
\caption{Model compression results on Penn Tree Bank dataset}\label{exp:robustlm}
\centerline{
\begin{tabular}{|c|c|c|c|c|c|}
\hline
& Original & Round & Round \& clip & SVD ($\mathit{rank}=128$)& SVD ($\mathit{rank}=64$)\\\hline\hline
Baseline &52.8 & 53.2 (+0.4) & 53.6 (+0.8) &56.6 (+3.8) & 65.5 (+12.7)  \\\hline
Sharpened Sigmoid &53.2  & 53.5 (+0.3) & 53.6 (\textbf{+0.4})  &54.6 (+1.4)& 60.0 (+6.8)\\\hline
$G^2$-LSTM &\textbf{52.1} &\textbf{52.2} (\textbf{+0.1})&\textbf{52.8} (+0.7) &\textbf{53.3} (\textbf{+1.2}) & \textbf{56.0} (\textbf{+3.9}) \\\hline
\end{tabular}
}
\end{table*}

\begin{table*}[htb]
\centering
\caption{Model compression results on IWSLT German$\rightarrow$English dataset}\label{exp:robust}
\centerline{
\begin{tabular}{|c|c|c|c|c|c|}
\hline
& Original & Round & Round \& clip  & SVD ($\mathit{rank}=32$)& SVD ($\mathit{rank}=16$)\\\hline\hline
Baseline &31.00 &28.65 (-2.35) & 21.97 (-9.03)  &30.52 (-0.48)& 29.56 (-1.44)\\\hline
Sharpened Sigmoid &29.73 &27.08 (-2.65) & 25.14 (-4.59)  &29.17 (-0.53)& 28.82 (-0.91)\\\hline
$G^2$-LSTM &\textbf{31.95} &\textbf{31.44} (\textbf{-0.51})& \textbf{31.44} (\textbf{-0.51})  &\textbf{31.62} (\textbf{-0.33})&\textbf{31.28} (\textbf{-0.67}) \\\hline
\end{tabular}
}
\end{table*}
\begin{table*}[htb]
\centering
\caption{ Model compression results on WMT English$\rightarrow$German dataset}\label{exp:robust2}
\centerline{
\begin{tabular}{|c|c|c|c|c|c|}
\hline
& Original & Round & Round \& clip     & SVD ($\mathit{rank}=32$)& SVD ($\mathit{rank}=16$)\\\hline\hline
Baseline &21.89 &16.22 (-5.67) & 16.03 (-5.86)  &21.15 (-0.74)&19.99 (-1.90)\\\hline
Sharpened Sigmoid &21.64 &16.85 (-4.79) & 16.72 (-4.92)  &20.98 (-0.66)&19.87 (-1.77)\\\hline
$G^2$-LSTM &\textbf{22.43}&\textbf{20.15} (\textbf{-2.28}) & \textbf{20.29} (\textbf{-2.14})&\textbf{22.16} (\textbf{-0.27})&\textbf{21.84} (\textbf{-0.51})\\\hline
\end{tabular}
}
\end{table*}

\subsection{Experimental Results}
The experimental results are shown in Table~\ref{exp:performancelm} and \ref{exp:performance}. We compare our training method with two algorithms.  For the first algorithm (we call it \emph{Baseline}), we remove the Gumble-Softmax trick and train the model using standard optimization methods. For the second algorithm (we call it \emph{Sharpened Sigmoid}), we use a sharpened sigmoid function as described in Section~\ref{ssec:training} by setting $\tau=0.2$ and check whether such trick can bring better performance.

From the results, we can see that our learned models are competitive or better than all baseline models. In language modeling task, we outperform the baseline algorithms for 0.7/1.1 points (1.2/1.4 points without continuous cache pointer) in terms of test perplexity. For machine translation, we outperform the baselines for 0.95/2.22 and 0.54/0.79 points in terms of BLEU score for German$\rightarrow$English and English$\rightarrow$German dataset respectively. Note that the only difference between $G^2$-LSTM and the baselines is the training algorithm, while they adopt the same model structure. Thus, better results of $G^2$-LSTM demonstrate the effectiveness of our proposed training method. This shows that restricting the outputs of the gates towards binary values doesn't bring performance drop at all. On the contrary, the performances are even better. We conclude that such benefit may come from the better generalization ability.

We also list the performance of previous works in literature, which may adopt different model architectures or settings. For language modeling, we obtain better performance results compared to the previous works listed in the table. For German$\rightarrow$English translation, the two-layer stacked encoder-decoder model we learned outperforms all previous works. For English$\rightarrow$German translation, our result is worse than GNMT \citep{wu2016google} as they used a stacked eight-layer model while we only used a three-layer one.

\subsection{Sensitivity Analysis}
We conducted a set of experiments to test how sensitive our learned models were when compressing their gate parameters. We considered two ways of compression as follows.
\begin{description}[leftmargin=0cm]
\item [Low-Precision Compression] We compressed parameters in the input and forget gates to lower precision. Doing so the model can be compressed to a relatively small size. In particular, we applied round and clip operations to the parameters of the input and forget gates:
    \begin{eqnarray}
    \text{round}_r(x)&=&\text{round}(x/r)*r, \label{eqn:round}\\
    \text{clip}_c(x)&=&\text{clip}(x,-c,c). \label{eqn:clip}
    \end{eqnarray}
    We tested two settings of low-precision compression. In the first setting (named as \emph{Round}), we rounded the parameters using Eqn.~(\ref{eqn:round}). In this way, we reduced the support set of the parameters in the gates. In the second setting (named as \emph{Round \& Clip}), we further clipped the rounded value to a fixed range using Eqn.~(\ref{eqn:clip}) and thus restricted the number of different values. As the two tasks are far different, we set the round parameter $r=0.2$ and the clip parameter $c=0.4$ for the task of language modeling, and set $c=1.0$ and $r=0.5$ for neural machine translation. As a result, parameters of input gates and forget gates in language modeling can only take values from ($0.0, \pm0.2, \pm0.4$), and ($0.0, \pm0.5, \pm1.0$) for machine translation.
\item [Low-Rank Compression] We compressed parameter matrices of the input/forget gates to lower-rank matrices through singular value decomposition, which can reduce the model size and lead to faster matrix multiplication. Given that the hidden states of the task of language modeling were of much larger dimension than that of neural machine translation, we set $\mathit{rank}=64/128$ for language modeling and $\mathit{rank}=16/32$ for neural machine translation.
\end{description}
We summarize the results in Table \ref{exp:robustlm}-\ref{exp:robust2}. From Table \ref{exp:robustlm}, we can see that for language modeling both the baseline and our learned model are quite robust to low-precision compression, but our model is much more robust and significantly outperforms the baseline with low-rank approximation. Even setting $\mathit{rank}=64$ (roughly 12$\times$ compression rate of the gates), we still get 56.0 perplexity, while the perplexity of the baseline model increases from 52.8 to 65.5, i.e., becoming 24\% worse. For machine translation, our proposed method is always better than the baseline model, no matter for low-precision or low-rank compression. Even if setting $\mathit{rank}=16$ (roughly 8$\times$/32$\times$ compression rate of the gates for German$\rightarrow$English and English$\rightarrow$German respectively), we still get roughly comparable translation accuracy to the baseline model with full parameters. All results show that the models trained with our proposed method are less sensitive to parameter compression.

\begin{figure*}[htb]
\centering
\begin{minipage}{0.30\linewidth}
\subfigure[Input gates in $G^2$-LSTM]{
\includegraphics[width = 1\linewidth]{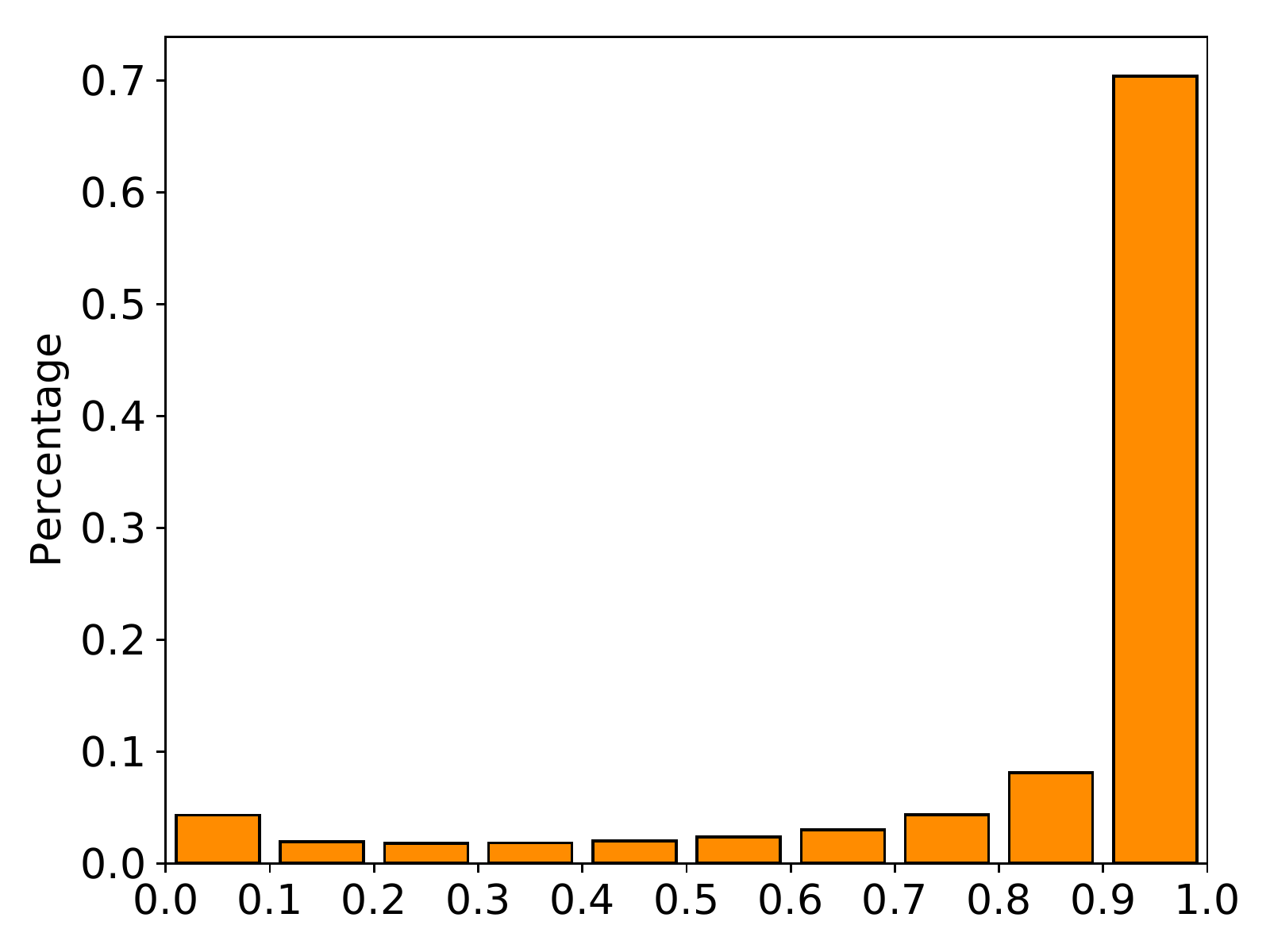}
}
\end{minipage}%
\qquad
\begin{minipage}{0.30\linewidth}
\subfigure[Forget gates in $G^2$-LSTM]{
\includegraphics[width = 1\linewidth]{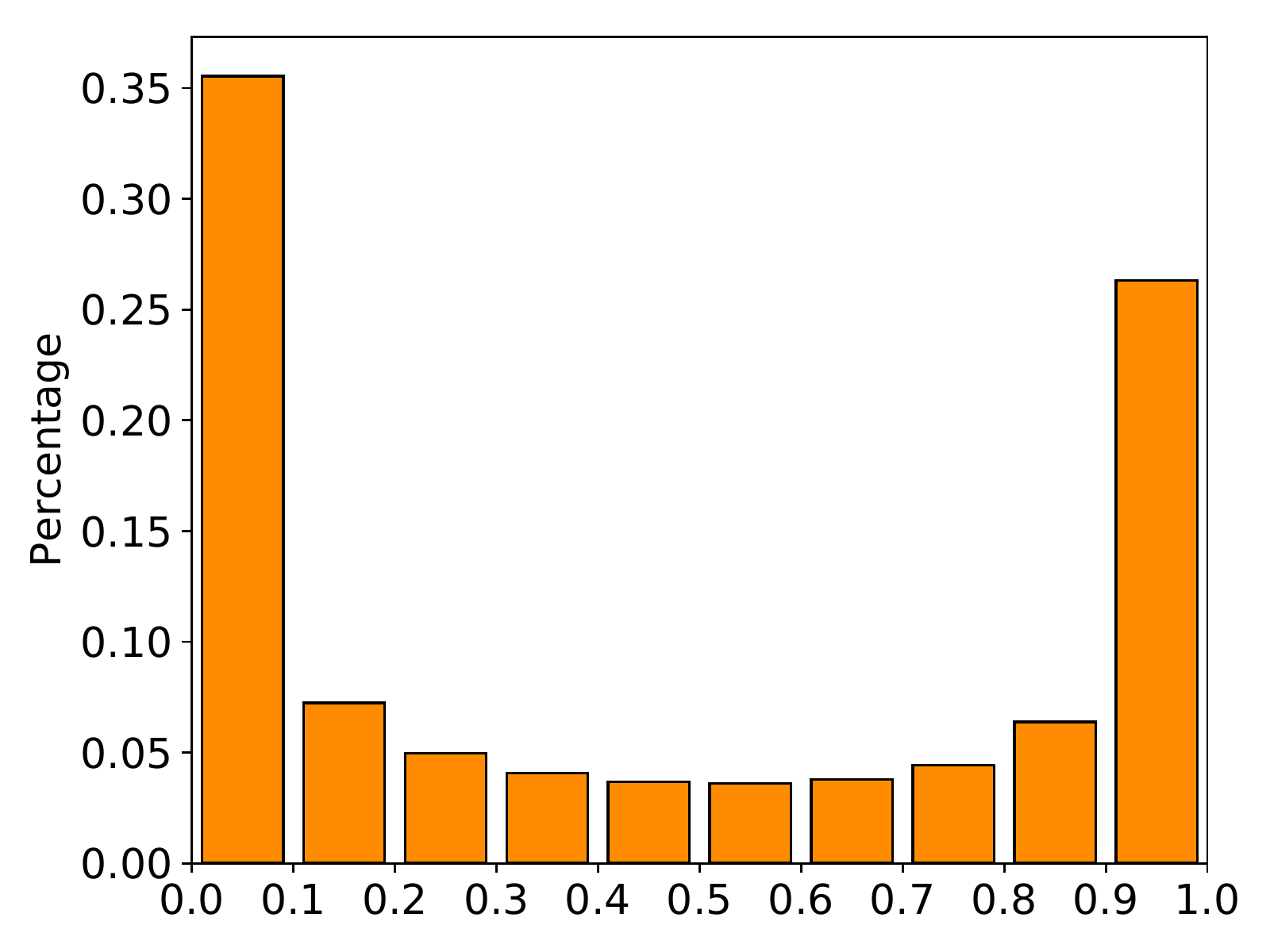}
}
\end{minipage}%
\caption{Histograms of gate value distributions in $G^2$-LSTM, from the same data as Figure~\ref{fig:gate_lstm}.}
\label{fig:gate_g2_lstm}
\end{figure*}

\begin{figure*}[tb]
    \centering
    \includegraphics[width = 0.8\linewidth]{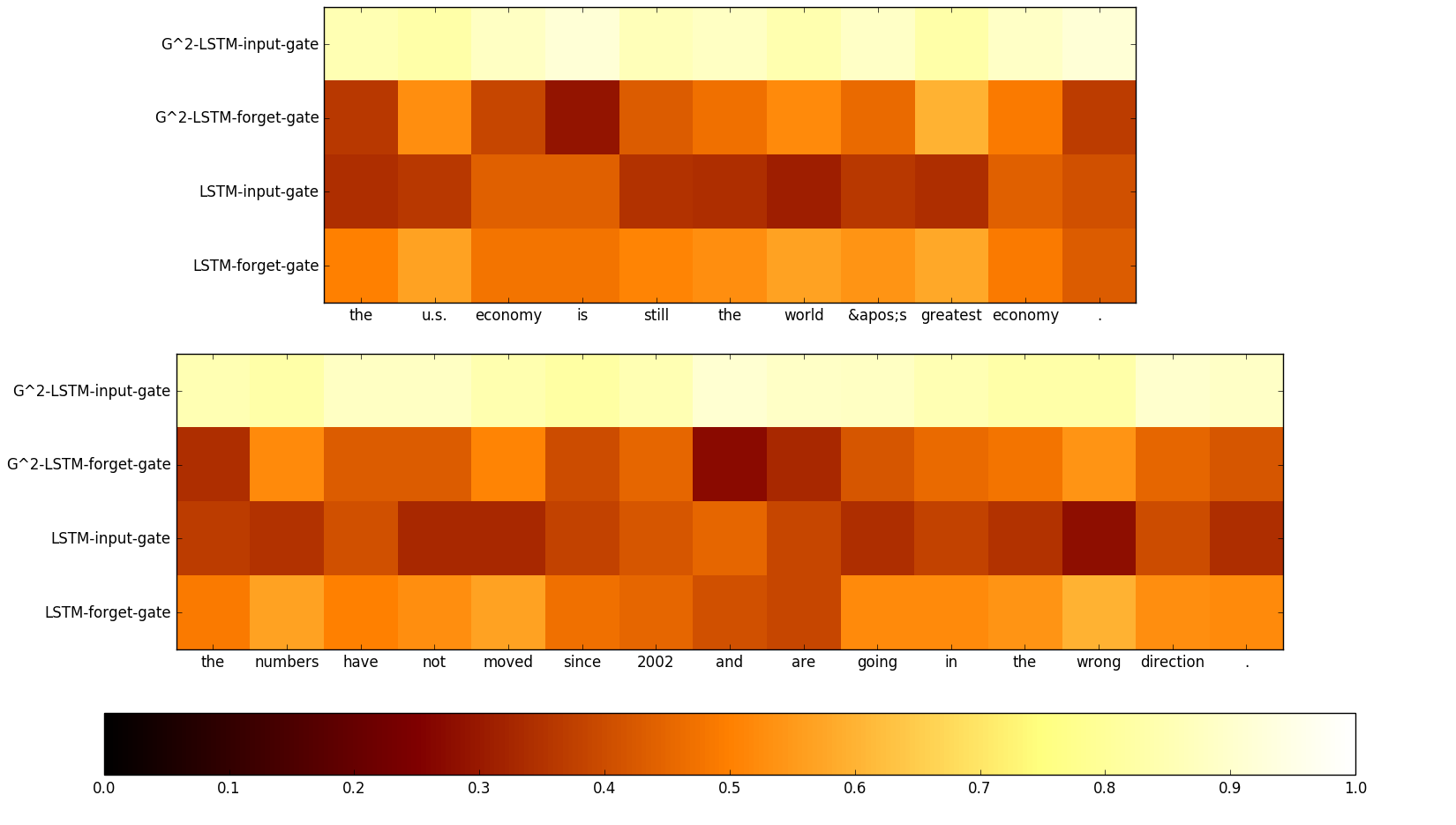}
    \caption{Visualization of average gate value at each timestep in LSTM and $G^2$-LSTM, from the same model as Figure~\ref{fig:gate_lstm}.}
    \label{fig:heatmaps}
\end{figure*}

\subsection{Visualization of the Gates}
In addition to comparing the final performances, we further looked inside the learned models and checked the gates.

To well verify the effectiveness of our proposed $G^2$-LSTM, we did a set of experiments to show the values of gates learned by $G^2$-LSTM are near the boundary and reasonable, based on the model learned from German$\rightarrow$English translation task. We show the value distribution of the gates trained using classic LSTM and $G^2$-LSTM. To achieve this, we sampled 10000 sentence pairs from the training set and fed them into the learned models. We got the output value vectors of the input/forget gates in the first layer of the decoder. We recorded the value of each element in the output vectors and plotted the distributions in Figure~\ref{fig:gate_lstm} and~\ref{fig:gate_g2_lstm}.

From the figures, we can see that although both LSTM and $G^2$-LSTM work reasonably well in practice, the output values of the gates are very different. In LSTM, the distributions of the gate values are relatively uniform and have no clear concentration. In contrast, the values of the input gates of $G^2$-LSTM are concentrated in the region close to 1, which suggests that our learned model tries to keep most information from the input words; the values of the forget gates are concentrated in the boundary regions (i.e., either the region close to 0 or the region close to 1). This observation shows that our training algorithm meets our expectation and successfully pushes the gates to 0/1.

Besides the overall distribution of gate values over a sampled set of training data, here we provide a case study for sampled sentences. We calculated the average value of the output vector of the input and forget gate functions for each word. In particular, we focused on the average value of the input/forget gate functions in the first layer and check whether the averages are reasonable. We plot the heatmap of the English sentence part in Figure \ref{fig:heatmaps}. 

First, we can see that our $G^2$-LSTM does not drop information in the input gate function since the average values are relatively large for all words. In contrast, the average values of the input gates of LSTM are sometimes small (less than 0.5), even for the meaningful word like ``wrong''. As those words are not included into LSTM, they cannot be effectively encoded and decoded, and thus lead to bad translation results. Second, for $G^2$-LSTM, most of the words with small values for forget gates are function words (e.g., conjunctions and punctuations) or the boundaries in clauses. That is, our training algorithm indeed ensures the model to forget information on the boundaries in the sentences, and reset the hidden states with new inputs.

\section{Conclusion and Future Work}
In this paper, we designed a new training algorithm for LSTM by leveraging the recently developed Gumbel-Softmax estimator. Our training algorithm can push the values of the input and forget gates to 0 or 1, leading to robust LSTM models. Experiments on language modeling and machine translation demonstrated the effectiveness of the proposed training algorithm.

We will explore following directions in the future. First, we will apply our algorithm to deeper models (e.g., 8+ layers) and test on larger datasets. Second, we have considered the tasks of language modeling and machine translation. We will study more applications such as question answering and text summarization.
\section*{Acknowledgements}
This work was partially supported by National Basic Research Program of China (973 Program) (grant no. 2015CB352502), NSFC (61573026). We would like to thank Chen Xing and Qizhe Xie for helpful discussions, and the anonymous reviewers for their valuable comments on our paper.

\bibliography{icml2018_conference}

\begin{thebibliography}{50}
\providecommand{\natexlab}[1]{#1}
\providecommand{\url}[1]{\texttt{#1}}
\expandafter\ifx\csname urlstyle\endcsname\relax
  \providecommand{\doi}[1]{doi: #1}\else
  \providecommand{\doi}{doi: \begingroup \urlstyle{rm}\Url}\fi

\bibitem[Bahdanau et~al.(2014)Bahdanau, Cho, and Bengio]{bahdanau2014neural}
Bahdanau, D., Cho, K., and Bengio, Y.
\newblock Neural machine translation by jointly learning to align and
  translate.
\newblock \emph{arXiv preprint arXiv:1409.0473}, 2014.

\bibitem[Bahdanau et~al.(2016)Bahdanau, Brakel, Xu, Goyal, Lowe, Pineau,
  Courville, and Bengio]{AC4SequencePrediction}
Bahdanau, D., Brakel, P., Xu, K., Goyal, A., Lowe, R., Pineau, J., Courville,
  A., and Bengio, Y.
\newblock An actor-critic algorithm for sequence prediction.
\newblock \emph{arXiv preprint arXiv:1607.07086}, 2016.

\bibitem[Britz et~al.(2017)Britz, Goldie, Luong, and Le]{britz2017massive}
Britz, D., Goldie, A., Luong, T., and Le, Q.
\newblock Massive exploration of neural machine translation architectures.
\newblock \emph{arXiv preprint arXiv:1703.03906}, 2017.

\bibitem[Cettolo et~al.(2014)Cettolo, Niehues, St{\"u}ker, Bentivogli, and
  Federico]{cettolo2014report}
Cettolo, M., Niehues, J., St{\"u}ker, S., Bentivogli, L., and Federico, M.
\newblock Report on the 11th iwslt evaluation campaign, iwslt 2014.
\newblock In \emph{Proceedings of the International Workshop on Spoken Language
  Translation, Hanoi, Vietnam}, 2014.

\bibitem[Chaudhari et~al.(2016)Chaudhari, Choromanska, Soatto, and
  LeCun]{chaudhari2016entropy}
Chaudhari, P., Choromanska, A., Soatto, S., and LeCun, Y.
\newblock Entropy-sgd: Biasing gradient descent into wide valleys.
\newblock \emph{arXiv preprint arXiv:1611.01838}, 2016.

\bibitem[Courbariaux et~al.(2016)Courbariaux, Hubara, Soudry, El-Yaniv, and
  Bengio]{courbariaux2016binarized}
Courbariaux, M., Hubara, I., Soudry, D., El-Yaniv, R., and Bengio, Y.
\newblock Binarized neural networks: Training deep neural networks with weights
  and activations constrained to+ 1 or-1.
\newblock \emph{arXiv preprint arXiv:1602.02830}, 2016.

\bibitem[Gal \& Ghahramani(2016)Gal and Ghahramani]{gal2016theoretically}
Gal, Y. and Ghahramani, Z.
\newblock A theoretically grounded application of dropout in recurrent neural
  networks.
\newblock In \emph{Advances in neural information processing systems}, pp.\
  1019--1027, 2016.

\bibitem[Gehring et~al.(2017)Gehring, Auli, Grangier, Yarats, and
  Dauphin]{gehring2017convolutional}
Gehring, J., Auli, M., Grangier, D., Yarats, D., and Dauphin, Y.~N.
\newblock Convolutional sequence to sequence learning.
\newblock \emph{arXiv preprint arXiv:1705.03122}, 2017.

\bibitem[Gers et~al.(1999)Gers, Schmidhuber, and Cummins]{gers1999learning}
Gers, F.~A., Schmidhuber, J., and Cummins, F.
\newblock Learning to forget: Continual prediction with lstm.
\newblock 1999.

\bibitem[Grave et~al.(2016)Grave, Joulin, and Usunier]{grave2016improving}
Grave, E., Joulin, A., and Usunier, N.
\newblock Improving neural language models with a continuous cache.
\newblock \emph{arXiv preprint arXiv:1612.04426}, 2016.

\bibitem[Haussler et~al.(1997)Haussler, Opper, et~al.]{haussler1997mutual}
Haussler, D., Opper, M., et~al.
\newblock Mutual information, metric entropy and cumulative relative entropy
  risk.
\newblock \emph{The Annals of Statistics}, 25\penalty0 (6):\penalty0
  2451--2492, 1997.

\bibitem[He et~al.(2016)He, Xia, Qin, Wang, Yu, Liu, and Ma]{he2016dual}
He, D., Xia, Y., Qin, T., Wang, L., Yu, N., Liu, T., and Ma, W.-Y.
\newblock Dual learning for machine translation.
\newblock In \emph{Advances in Neural Information Processing Systems}, pp.\
  820--828, 2016.

\bibitem[He et~al.(2017)He, Lu, Xia, Qin, Wang, and Liu]{he2017decoding}
He, D., Lu, H., Xia, Y., Qin, T., Wang, L., and Liu, T.
\newblock Decoding with value networks for neural machine translation.
\newblock In \emph{Advances in Neural Information Processing Systems}, pp.\
  177--186, 2017.

\bibitem[Hochreiter(1998)]{hochreiter1998vanishing}
Hochreiter, S.
\newblock The vanishing gradient problem during learning recurrent neural nets
  and problem solutions.
\newblock \emph{International Journal of Uncertainty, Fuzziness and
  Knowledge-Based Systems}, 6\penalty0 (02):\penalty0 107--116, 1998.

\bibitem[Hochreiter \& Schmidhuber(1997{\natexlab{a}})Hochreiter and
  Schmidhuber]{hochreiter1997flat}
Hochreiter, S. and Schmidhuber, J.
\newblock Flat minima.
\newblock \emph{Neural Computation}, 9\penalty0 (1):\penalty0 1--42,
  1997{\natexlab{a}}.

\bibitem[Hochreiter \& Schmidhuber(1997{\natexlab{b}})Hochreiter and
  Schmidhuber]{hochreiter1997long}
Hochreiter, S. and Schmidhuber, J.
\newblock Long short-term memory.
\newblock \emph{Neural computation}, 9\penalty0 (8):\penalty0 1735--1780,
  1997{\natexlab{b}}.

\bibitem[Huang et~al.()Huang, Wang, Zhou, and Deng]{huangtoward}
Huang, P.-S., Wang, C., Zhou, D., and Deng, L.
\newblock Toward neural phrase-based machine translation.

\bibitem[Inan et~al.(2016)Inan, Khosravi, and Socher]{inan2016tying}
Inan, H., Khosravi, K., and Socher, R.
\newblock Tying word vectors and word classifiers: A loss framework for
  language modeling.
\newblock \emph{arXiv preprint arXiv:1611.01462}, 2016.

\bibitem[Jang et~al.(2016)Jang, Gu, and Poole]{jang2016categorical}
Jang, E., Gu, S., and Poole, B.
\newblock Categorical reparameterization with gumbel-softmax.
\newblock \emph{arXiv preprint arXiv:1611.01144}, 2016.

\bibitem[Jean et~al.(2015)Jean, Cho, Memisevic, and Bengio]{chousing}
Jean, S., Cho, K., Memisevic, R., and Bengio, Y.
\newblock On using very large target vocabulary for neural machine translation.
\newblock In \emph{Proceedings of the 53rd Annual Meeting of the Association
  for Computational Linguistics and the 7th International Joint Conference on
  Natural Language Processing (Volume 1: Long Papers)}, pp.\  1--10, Beijing,
  China, July 2015. Association for Computational Linguistics.
\newblock URL \url{http://www.aclweb.org/anthology/P15-1001}.

\bibitem[Jozefowicz et~al.(2016)Jozefowicz, Vinyals, Schuster, Shazeer, and
  Wu]{jozefowicz2016exploring}
Jozefowicz, R., Vinyals, O., Schuster, M., Shazeer, N., and Wu, Y.
\newblock Exploring the limits of language modeling.
\newblock \emph{arXiv preprint arXiv:1602.02410}, 2016.

\bibitem[Karpathy et~al.(2015)Karpathy, Johnson, and
  Fei-Fei]{karpathy2015visualizing}
Karpathy, A., Johnson, J., and Fei-Fei, L.
\newblock Visualizing and understanding recurrent networks.
\newblock \emph{arXiv preprint arXiv:1506.02078}, 2015.

\bibitem[Keskar et~al.(2016)Keskar, Mudigere, Nocedal, Smelyanskiy, and
  Tang]{keskar2016large}
Keskar, N.~S., Mudigere, D., Nocedal, J., Smelyanskiy, M., and Tang, P. T.~P.
\newblock On large-batch training for deep learning: Generalization gap and
  sharp minima.
\newblock \emph{arXiv preprint arXiv:1609.04836}, 2016.

\bibitem[Kim et~al.(2016)Kim, Jernite, Sontag, and Rush]{kim2016character}
Kim, Y., Jernite, Y., Sontag, D., and Rush, A.~M.
\newblock Character-aware neural language models.
\newblock In \emph{AAAI}, pp.\  2741--2749, 2016.

\bibitem[Krueger et~al.(2016)Krueger, Maharaj, Kramar, Pezeshki, Ballas, Ke,
  Goyal, Bengio, Courville, and Pal]{krueger2016zoneout}
Krueger, D., Maharaj, T., Kramar, J., Pezeshki, M., Ballas, N., Ke, N.~R.,
  Goyal, A., Bengio, Y., Courville, A., and Pal, C.
\newblock Zoneout: Regularizing rnns by randomly preserving hidden activations.
\newblock 2016.

\bibitem[Kusner \& Hern{\'a}ndez-Lobato(2016)Kusner and
  Hern{\'a}ndez-Lobato]{kusner2016gans}
Kusner, M.~J. and Hern{\'a}ndez-Lobato, J.~M.
\newblock Gans for sequences of discrete elements with the gumbel-softmax
  distribution.
\newblock \emph{arXiv preprint arXiv:1611.04051}, 2016.

\bibitem[Luong et~al.(2015)Luong, Pham, and Manning]{luong2015effective}
Luong, M.-T., Pham, H., and Manning, C.~D.
\newblock Effective approaches to attention-based neural machine translation.
\newblock \emph{arXiv preprint arXiv:1508.04025}, 2015.

\bibitem[Maddison et~al.(2016)Maddison, Mnih, and Teh]{maddison2016concrete}
Maddison, C.~J., Mnih, A., and Teh, Y.~W.
\newblock The concrete distribution: A continuous relaxation of discrete random
  variables.
\newblock \emph{arXiv preprint arXiv:1611.00712}, 2016.

\bibitem[Melis et~al.(2017)Melis, Dyer, and Blunsom]{melis2017state}
Melis, G., Dyer, C., and Blunsom, P.
\newblock On the state of the art of evaluation in neural language models.
\newblock \emph{arXiv preprint arXiv:1707.05589}, 2017.

\bibitem[Merity et~al.(2016)Merity, Xiong, Bradbury, and
  Socher]{merity2016pointer}
Merity, S., Xiong, C., Bradbury, J., and Socher, R.
\newblock Pointer sentinel mixture models.
\newblock \emph{arXiv preprint arXiv:1609.07843}, 2016.

\bibitem[Merity et~al.(2017)Merity, Keskar, and Socher]{merity2017regularizing}
Merity, S., Keskar, N.~S., and Socher, R.
\newblock Regularizing and optimizing lstm language models.
\newblock \emph{arXiv preprint arXiv:1708.02182}, 2017.

\bibitem[Murdoch \& Szlam(2017)Murdoch and Szlam]{murdoch2017automatic}
Murdoch, W.~J. and Szlam, A.
\newblock Automatic rule extraction from long short term memory networks.
\newblock \emph{arXiv preprint arXiv:1702.02540}, 2017.

\bibitem[Polyak \& Juditsky(1992)Polyak and Juditsky]{polyak1992acceleration}
Polyak, B.~T. and Juditsky, A.~B.
\newblock Acceleration of stochastic approximation by averaging.
\newblock \emph{SIAM Journal on Control and Optimization}, 30\penalty0
  (4):\penalty0 838--855, 1992.

\bibitem[Ranzato et~al.(2015)Ranzato, Chopra, Auli, and Zaremba]{PG4Sequence}
Ranzato, M., Chopra, S., Auli, M., and Zaremba, W.
\newblock Sequence level training with recurrent neural networks.
\newblock \emph{arXiv preprint arXiv:1511.06732}, 2015.

\bibitem[Sennrich et~al.(2016)Sennrich, Haddow, and Birch]{BPE}
Sennrich, R., Haddow, B., and Birch, A.
\newblock Neural machine translation of rare words with subword units.
\newblock In \emph{ACL}, 2016.

\bibitem[Shen et~al.(2015)Shen, Cheng, He, He, Wu, Sun, and
  Liu]{shen2015minimum}
Shen, S., Cheng, Y., He, Z., He, W., Wu, H., Sun, M., and Liu, Y.
\newblock Minimum risk training for neural machine translation.
\newblock \emph{arXiv preprint arXiv:1512.02433}, 2015.

\bibitem[Subramanian et~al.(2017)Subramanian, Rajeswar, Dutil, Pal, and
  Courville]{subramanian2017adversarial}
Subramanian, S., Rajeswar, S., Dutil, F., Pal, C., and Courville, A.
\newblock Adversarial generation of natural language.
\newblock \emph{ACL 2017}, pp.\  241, 2017.

\bibitem[Villegas et~al.(2017)Villegas, Yang, Zou, Sohn, Lin, and
  Lee]{villegas2017learning}
Villegas, R., Yang, J., Zou, Y., Sohn, S., Lin, X., and Lee, H.
\newblock Learning to generate long-term future via hierarchical prediction.
\newblock \emph{arXiv preprint arXiv:1704.05831}, 2017.

\bibitem[Vinyals et~al.(2015)Vinyals, Toshev, Bengio, and
  Erhan]{vinyals2015show}
Vinyals, O., Toshev, A., Bengio, S., and Erhan, D.
\newblock Show and tell: A neural image caption generator.
\newblock In \emph{Proceedings of the IEEE conference on computer vision and
  pattern recognition}, pp.\  3156--3164, 2015.

\bibitem[Wan et~al.(2013)Wan, Zeiler, Zhang, Le~Cun, and
  Fergus]{wan2013regularization}
Wan, L., Zeiler, M., Zhang, S., Le~Cun, Y., and Fergus, R.
\newblock Regularization of neural networks using dropconnect.
\newblock In \emph{International Conference on Machine Learning}, pp.\
  1058--1066, 2013.

\bibitem[Wiseman \& Rush(2016{\natexlab{a}})Wiseman and Rush]{BSO}
Wiseman, S. and Rush, A.~M.
\newblock Sequence-to-sequence learning as beam-search optimization.
\newblock In \emph{EMNLP}, November 2016{\natexlab{a}}.

\bibitem[Wiseman \& Rush(2016{\natexlab{b}})Wiseman and
  Rush]{wiseman2016sequence}
Wiseman, S. and Rush, A.~M.
\newblock Sequence-to-sequence learning as beam-search optimization.
\newblock \emph{arXiv preprint arXiv:1606.02960}, 2016{\natexlab{b}}.

\bibitem[Wu et~al.(2016)Wu, Schuster, Chen, Le, Norouzi, Macherey, Krikun, Cao,
  Gao, Macherey, et~al.]{wu2016google}
Wu, Y., Schuster, M., Chen, Z., Le, Q.~V., Norouzi, M., Macherey, W., Krikun,
  M., Cao, Y., Gao, Q., Macherey, K., et~al.
\newblock Google's neural machine translation system: Bridging the gap between
  human and machine translation.
\newblock \emph{arXiv preprint arXiv:1609.08144}, 2016.

\bibitem[Xingjian et~al.(2015)Xingjian, Chen, Wang, Yeung, Wong, and
  Woo]{xingjian2015convolutional}
Xingjian, S., Chen, Z., Wang, H., Yeung, D.-Y., Wong, W.-K., and Woo, W.-c.
\newblock Convolutional lstm network: A machine learning approach for
  precipitation nowcasting.
\newblock In \emph{Advances in neural information processing systems}, pp.\
  802--810, 2015.

\bibitem[Xu et~al.(2015)Xu, Ba, Kiros, Cho, Courville, Salakhudinov, Zemel, and
  Bengio]{xu2015show}
Xu, K., Ba, J., Kiros, R., Cho, K., Courville, A., Salakhudinov, R., Zemel, R.,
  and Bengio, Y.
\newblock Show, attend and tell: Neural image caption generation with visual
  attention.
\newblock In \emph{International Conference on Machine Learning}, pp.\
  2048--2057, 2015.

\bibitem[Zaremba et~al.(2014)Zaremba, Sutskever, and
  Vinyals]{zaremba2014recurrent}
Zaremba, W., Sutskever, I., and Vinyals, O.
\newblock Recurrent neural network regularization.
\newblock \emph{arXiv preprint arXiv:1409.2329}, 2014.

\bibitem[Zeiler(2012)]{zeiler2012adadelta}
Zeiler, M.~D.
\newblock Adadelta: an adaptive learning rate method.
\newblock \emph{arXiv preprint arXiv:1212.5701}, 2012.

\bibitem[Zhang et~al.(2016)Zhang, Chen, Yu, Yaco, Khudanpur, and
  Glass]{zhang2016highway}
Zhang, Y., Chen, G., Yu, D., Yaco, K., Khudanpur, S., and Glass, J.
\newblock Highway long short-term memory rnns for distant speech recognition.
\newblock In \emph{Acoustics, Speech and Signal Processing (ICASSP), 2016 IEEE
  International Conference on}, pp.\  5755--5759. IEEE, 2016.

\bibitem[Zilly et~al.(2016)Zilly, Srivastava, Koutn{\'\i}k, and
  Schmidhuber]{zilly2016recurrent}
Zilly, J.~G., Srivastava, R.~K., Koutn{\'\i}k, J., and Schmidhuber, J.
\newblock Recurrent highway networks.
\newblock \emph{arXiv preprint arXiv:1607.03474}, 2016.

\bibitem[Zoph \& Le(2016)Zoph and Le]{zoph2016neural}
Zoph, B. and Le, Q.~V.
\newblock Neural architecture search with reinforcement learning.
\newblock \emph{arXiv preprint arXiv:1611.01578}, 2016.

\end{thebibliography}
\bibliographystyle{icml2018}
\end{document}